\documentclass[11pt]{article}

\usepackage[margin=1in]{geometry}

\usepackage{graphicx}
\usepackage{mathrsfs}
\usepackage{amsmath}
\usepackage{bm}
\usepackage{amsfonts}
\usepackage{dsfont}
\usepackage{mathtools}
\usepackage{amsthm}
\usepackage{enumitem}
\usepackage{algorithm}
\usepackage{algpseudocode}
\usepackage{tikz}
\usetikzlibrary{arrows.meta,calc}

\numberwithin{equation}{section}
\allowdisplaybreaks

\usepackage[round]{natbib}
\usepackage{xcolor}
\definecolor{linkblue}{RGB}{35,75,145}
\definecolor{citegreen}{RGB}{20,105,75}
\usepackage{hyperref}
\hypersetup{
    colorlinks=true,
    linkcolor=linkblue,
    citecolor=citegreen,
    urlcolor=linkblue,
    pdftitle={Equilibrium Learning in Unbounded Stochastic Games with Gaussian Transitions},
    pdfauthor={Jonathan Anthony Mavroforas and Fred Roosta}
}
\usepackage[capitalise]{cleveref}
\crefname{equation}{}{}
\crefname{figure}{Figure}{Figures}
\creflabelformat{equation}{\textup{(#2#1#3)}}
\crefname{assumption}{Assumption}{Assumptions}

\newcommand{\bb}[1]{\mathbb{#1}}
\newcommand{\linfty}[1]{\left\lVert #1 \right\rVert_{\infty}}

\newtheorem{theorem}{Theorem}[section]
\newtheorem{proposition}[theorem]{Proposition}
\newtheorem{lemma}[theorem]{Lemma}

\theoremstyle{definition}
\newtheorem{definition}{Definition}[section]
\newtheorem{assumption}{Assumption}

\theoremstyle{remark}
\newtheorem{remark}{Remark}[section]

\setlist[enumerate]{leftmargin=*,itemsep=3pt,topsep=3pt}
\setlist[itemize]{leftmargin=*,itemsep=3pt,topsep=3pt}

\title{\textbf{Equilibrium Learning in Unbounded Stochastic Games with Gaussian Transitions}}
\author{
Jonathan Anthony Mavroforas \qquad Fred Roosta\\[0.5em]
\small School of Mathematics and Physics\\
\small The University of Queensland, Brisbane, Australia
}
\date{}

\begin{document}

\maketitle

\begin{abstract}
Learning equilibria from data is difficult when several decision makers interact indefinitely and the system state has no prescribed bounds. We examine discounted general-sum stochastic games with compact action sets and an unknown nonlinear drift observed through Gaussian transition noise. The proposed procedure uses a single uninterrupted trajectory to update Gaussian-process models of the drift and to solve successively refined finite games. We prove a perturbation result showing that stability of equilibrium selection, together with uniform KL accuracy of Gaussian transition models, is sufficient for regional uniform convergence of the induced values and approximate Nash behavior in the original game. The region specifies where the guarantees are evaluated; strategies and dynamics continue to act on the entire Euclidean state space. For GP estimators, simultaneous confidence bounds yield a rate controlled by posterior uncertainty, kernel information gain, and joint-state dimension. Numerical studies illustrate how kernel choice and dimension appear in the predicted convergence behavior.
\end{abstract}


\section{Introduction}
\label{sec:introduction}

Multi-agent reinforcement learning (MARL) provides a framework for sequential decision making among interacting agents~\cite{albrecht2024multi,gronauer2022multi}. Despite strong empirical progress, rigorous guarantees remain limited for general-sum stochastic games with continuous spaces and unknown nonlinear dynamics. The difficulty is especially pronounced when the state space is unbounded and data arrive along one continuing trajectory: transition estimation, equilibrium computation, and dependence in the collected data must be treated jointly.

Two theoretical directions address complementary parts of this problem. Gaussian-process (GP) reinforcement learning provides calibrated uncertainty quantification for unknown nonlinear dynamics, including nonepisodic single-agent learning~\cite{sukhija2024neorl}. Separately, finite-game approximation establishes approximate Nash equilibria for continuous-state stochastic games with known transitions~\cite{saldi2024existence}. We first derive a KL-to-value convergence result for learned Gaussian transition kernels and then specialize it to coordinate-wise GP estimators, thereby characterizing sufficient selection-stability and statistical conditions for equilibrium convergence under unknown dynamics.

\paragraph{Contributions.}
\begin{enumerate}[leftmargin=*,itemsep=3pt,topsep=3pt]
    \item \textbf{Single-trajectory equilibrium learning.}
    We introduce GP-MARL (Algorithm~\ref{alg:finite_game_rl}), a model-based algorithm with centralized training and decentralized execution for infinite-horizon discounted general-sum games with unbounded continuous states and compact continuous actions. GP-MARL alternates between updating a GP transition model and solving a finite approximation of the learned game, requiring neither episodic resets nor a generative model.

    \item \textbf{Convergence theory for learned Gaussian dynamics.}
    Theorem~\ref{theorem:convergence_rate} establishes that local policy-selection stability and uniform KL convergence of learned Gaussian transition kernels jointly guarantee uniform value convergence over any prescribed compact set of states. Combined with the finite-game approximation in Proposition~\ref{prop:approximation}, this convergence yields an approximate Nash equilibrium for the original continuous game. The compact set restricts only where the guarantee is stated; policies, deviations, Bellman dynamics, and trajectories remain defined on the full unbounded state space.

    \item \textbf{Verifiable GP conditions and rate.}
    Lemma~\ref{lemma:well_calibrated_gp} provides time-uniform calibrated confidence sets for the unknown transition dynamics. The GP specialization of Theorem~\ref{theorem:convergence_rate} then yields a high-probability value-error rate that holds across all learning iterations and depends explicitly on posterior contraction, maximum information gain, and joint state dimension. Together, these results provide concrete sufficient conditions for convergence in the single-trajectory, unknown-dynamics setting.
\end{enumerate}

\section{Related Work}
\label{sec:related_works}

\paragraph{Deep MARL.}
Value-decomposition methods such as VDN, QMIX, and QTRAN~\cite{sunehag2017value,rashid2020monotonic,son2019qtran}, and CTDE methods such as MADDPG, COMA, and MAPPO~\cite{lowe2017multi,foerster2018counterfactual,yu2022surprising}, achieve strong empirical performance~\cite{samvelyan2019starcraft,papoudakis2020benchmarking}. GP-MARL complements this literature with an all-iteration high-probability value-error analysis for continuous general-sum games learned from one continuing trajectory.

\paragraph{GP Sequential Decision Making.}
GP confidence sets yield regret and finite-sample guarantees in kernelized bandits and model-based reinforcement learning, with complexity governed by maximum information gain~\cite{srinivas2009gaussian,chowdhury2017kernelized,chowdhury2019online,curi2020efficient}. Recent work treats nonepisodic nonlinear systems and finite-horizon GP planning~\cite{sukhija2024neorl,bayrooti2026no}. These results address single-agent planning; in stochastic games, model error also perturbs the selected equilibrium. Our analysis separates this equilibrium-selection effect from transition-estimation error.

\paragraph{Continuous-State Stochastic Games.}
Existence results for discounted games with uncountable states rely on measurability, continuity, compactness, and absolute continuity~\cite{nowak1985existence}. Finite state--action approximations provide constructive approximate equilibria for standard Borel games~\cite{saldi2024existence}, assuming known transitions. GP-MARL couples this approximation theory with GP transition learning and quantifies the resulting value error.

\section{Problem Formulation}
\label{sec:problem_formulation}
We formulate the continuous stochastic game, identify conditions under which approximate equilibria exist, and specify the statistical assumptions used to learn its unknown dynamics. We then state the finite-game approximation result that connects equilibrium computation in finite games to the original uncountable game. Together, these ingredients provide the mathematical foundation for GP-MARL and its convergence analysis.

\subsection{Continuous stochastic games and equilibrium existence}

\begin{definition}[General-Sum Stochastic Game~\cite{albrecht2024multi}]
\label{def:general_sum_stochastic_game}
A general-sum stochastic game is a tuple
\[
G=\bigl(S,(A_i)_{i=1}^n,(\mathcal R_i)_{i=1}^n,\mathcal T,\mu,\gamma\bigr),
\]
where $[n]=\{1,\ldots,n\}$ is the finite set of agents; $S\subseteq\mathbb R^{n_S}$ is the measurable state space, with $n_S=\sum_{i=1}^n d_i$; $A_i\subseteq\mathbb R^{k_i}$ is agent $i$'s measurable action space and $A\coloneqq\prod_{i=1}^n A_i$; $\mathcal R_i:S\times A\to\mathbb R$ is agent $i$'s reward; $\mathcal T(\cdot\mid s,a)$ is a transition probability measure on $S$; $\mu$ is the initial-state distribution; and $\gamma\in[0,1)$ is the discount factor.
\end{definition}

The following assumptions place the game within the equilibrium-existence framework of~\citet{nowak1985existence} and~\citet{saldi2024existence}. We recall the regularity notion used below.

\begin{definition}[Carath\'eodory Function]
\label{def:caratheodory}
Let $(U,\mathcal U)$ be a measurable space and let $V$ be a metric space. A function $h:U\times V\to\mathbb R$ is a Carath\'eodory function if $h(u,\cdot)$ is continuous for every $u\in U$ and $h(\cdot,v)$ is measurable for every $v\in V$.
\end{definition}

\begin{assumption}[Structural Assumptions]
\label{assumption:structural_assumptions}
We assume that:
\begin{enumerate}[label={\bfseries (\alph*)}]
    \item $S=\mathbb R^{n_S}$ is equipped with its Borel $\sigma$-algebra $\mathcal S$, which is countably generated, and the game has no terminal states;
    \item \label{assumption:structural_assumptions:action_compact} each $A_i$ is a non-empty compact metric space equipped with its Borel $\sigma$-algebra;
    \item each reward function $\mathcal R_i$ is bounded;
    \item $\gamma\in[0,1)$;
    \item $\mathcal T$ is absolutely continuous with respect to a reference probability measure $\mathbb P$, with product-measurable Radon--Nikodym derivative $b:S\times A\times S\to\mathbb R_+$ satisfying
    \[
    \begin{aligned}
        \mathcal T(B\mid s,a)
        &=\int_B b(s,a,s')\,\mathbb P(ds'),\\
        \int_S\sup_{a\in A}b(s,a,s')\,\mathbb P(ds')&<\infty,
    \end{aligned}
    \]
    for every $B\in\mathcal S$ and $s\in S$, and $b(s,\cdot,\cdot)$ is a Carath\'eodory function.
\end{enumerate}
\end{assumption}

\begin{assumption}[Fully Observable Environment]
\label{assumption:fully_observable}
At time $t$, the agents have access to the complete state--action history up to time $t$.
\end{assumption}

For a joint policy $\pi=(\pi_1,\ldots,\pi_n)$, the value of agent $i$ is
\[
V_i^\pi(s)
\coloneqq
\mathbb E_\pi\!\left[
\sum_{t=0}^{\infty}\gamma^t\mathcal R_i(s_t,a_t)
\,\middle|\,s_0=s
\right].
\]
Bounded rewards and $\gamma<1$ ensure that $V_i^\pi$ is well defined and bounded. It satisfies
\begin{equation}
\label{eq:value_function}
\begin{aligned}
V_i^\pi(s)
&=
\int_A\pi(da\mid s)\Bigl[
\mathcal R_i(s,a)\\
&\hspace{5.5em}+\gamma\int_S V_i^\pi(s')\,\mathcal T(ds'\mid s,a)
\Bigr].
\end{aligned}
\end{equation}
The associated action-value function, used in the stability proofs, is
\[
Q_i^\pi(s,a)
\coloneqq
\mathbb E_\pi\!\left[
\sum_{t=0}^{\infty}\gamma^t\mathcal R_i(s_t,a_t)
\,\middle|\,s_0=s,\ a_0=a
\right],
\]
with Bellman equation
\begin{equation}
\label{eq:action_value_function}
Q_i^\pi(s,a)
=
\mathcal R_i(s,a)
+\gamma\int_S V_i^\pi(s')\,\mathcal T(ds'\mid s,a).
\end{equation}

\begin{definition}[$\varepsilon$-Nash Equilibrium~\cite{nowak1985existence,albrecht2024multi}]
\label{def:epsilon_nash_equilibrium}
For $\varepsilon\geq0$, a joint policy $\pi^\varepsilon=(\pi_1^\varepsilon,\ldots,\pi_n^\varepsilon)$ is an $\varepsilon$-Nash equilibrium if, for every $i\in[n]$, every state $s\in S$, and every admissible unilateral deviation $\pi_i$,
\[
V_i^{\pi^\varepsilon}(s)
\geq
V_i^{(\pi_i,\pi_{-i}^\varepsilon)}(s)-\varepsilon.
\]
The case $\varepsilon=0$ recovers a Nash equilibrium.
For a compact $\mathcal C\subset S$, an $\varepsilon$-Nash equilibrium \emph{on $\mathcal C$} satisfies the same inequality for every initial state $s\in\mathcal C$; the policy, deviations, transition law, and subsequent trajectory remain defined on the full state space $S$.
\end{definition}

\begin{proposition}[Existence of $\varepsilon$-Nash Equilibria~\cite{nowak1985existence}]
\label{prop:existence_of_eps_nash_equilibrium}
Suppose Assumptions~\ref{assumption:structural_assumptions} and~\ref{assumption:fully_observable} hold. Then, for every $\varepsilon>0$, the game admits an $\varepsilon$-Nash equilibrium policy $\pi_\varepsilon$.
\end{proposition}

This result establishes that the equilibrium-learning objective studied below is well posed. We next verify that our nonlinear Gaussian transition model belongs to this class.

\subsection{Unknown nonlinear dynamics}

\begin{assumption}[State Transition Dynamics]
\label{assumption:state_transition_dynamics}
The state evolves according to
\begin{equation}
\label{eq:state_transition_definition}
s_{t+1}=f^*(s_t,a_t)+w_t,
\qquad
w_t\sim\mathcal N(0,\sigma^2I_{n_S}),
\end{equation}
where $f^*:S\times A\to S$ is unknown and $(w_t)_{t\geq 0}$ is an independent and identically distributed noise sequence. Moreover, for every $s\in S$, there exists a non-negative $g_s\in L^1(S)$ such that
\begin{equation}
\label{eq:f_bounded_condition}
\sup_{a\in A}\mathcal N\!\left(s'\mid f^*(s,a),\sigma^2I_{n_S}\right)
\leq g_s(s'),
\qquad s'\in S.
\end{equation}
\end{assumption}

Condition~\ref{eq:f_bounded_condition} supplies the integrable domination required by Assumption~\ref{assumption:structural_assumptions}(e). Compactness of the action spaces makes this condition a direct consequence of continuity.

\begin{lemma}[Integrable Domination of Gaussian Transition Densities]
\label{lemma:existence_of_bound}
If $f^*:S\times A\to\mathbb R^{n_S}$ is continuous, then the Gaussian transition densities in~\cref{eq:state_transition_definition} satisfy~\cref{eq:f_bounded_condition}.
\end{lemma}
\begin{proof}
Fix $s\in S$. Continuity of $f^*(s,\cdot)$ and compactness of $A$ imply that
\[
K_s\coloneqq\{f^*(s,a):a\in A\}
\]
is compact, so $K_s\subseteq B_{r_s}(0)$ for some $r_s>0$. Hence
$\|s'-f^*(s,a)\|_2\geq\|s'\|_2-r_s$ for every $a\in A$. Writing $c_\sigma=(2\pi\sigma^2)^{-n_S/2}$, define
\[
g_s(s')
\coloneqq
c_\sigma\exp\!\left[
-\frac{(\|s'\|_2-r_s)_+^2}{2\sigma^2}
\right].
\]
This function is integrable on $\mathbb R^{n_S}$ and the reverse triangle inequality gives
\[
\sup_{a\in A}\mathcal N\!\left(s'\mid f^*(s,a),\sigma^2I_{n_S}\right)
\leq g_s(s'),
\]
which proves~\cref{eq:f_bounded_condition}.
\end{proof}

\begin{assumption}[Continuity]
\label{assumption:f_is_continuous}
The transition function $f^*$ is continuous on $S\times A$.
\end{assumption}

\begin{proposition}[Compatibility with the Existence Framework]
\label{prop:structural_assumptions_satisfied}
Suppose Assumption~\ref{assumption:structural_assumptions}(a--d) and Assumptions~\ref{assumption:state_transition_dynamics}--\ref{assumption:f_is_continuous} hold. Then Assumption~\ref{assumption:structural_assumptions}(e) also holds. Consequently, the stochastic game induced by~\cref{eq:state_transition_definition} admits an $\varepsilon$-Nash equilibrium for every $\varepsilon>0$.
\end{proposition}
\noindent\emph{Proof.} See Appendix~\ref{app:proof_structural_assumptions}.

Thus, the assumed dynamics simultaneously support equilibrium existence and the finite-game approximation introduced below.

\subsection{Calibrated Gaussian-process model}

Because $f^*$ is unknown, the analysis requires confidence sets that contain it uniformly over inputs, output coordinates, and learning iterations. Let $Z\coloneqq S\times A$.

\begin{definition}[All-Time Calibrated Confidence Sequence]
\label{def:confidence_set}
\label{def:all_time_valid_confidence_sequence}
For $\delta\in(0,1)$ and $N\geq0$, let $\mu_N:Z\to\mathbb R^{n_S}$, $\sigma_N:Z\to\mathbb R_+^{n_S}$, and $\beta_N(\delta)>0$. Define
\[
\begin{aligned}
\mathcal M_N(\delta)
&\coloneqq
\Bigl\{g:Z\to\mathbb R^{n_S}:\\[-0.3em]
&\quad |\mu_{N,j}(z)-g_j(z)|
\leq\beta_N(\delta)\sigma_{N,j}(z),\\[-0.3em]
&\quad \forall z\in Z,\ j\in[n_S]\Bigr\}.
\end{aligned}
\]
The sequence $(\mathcal M_N(\delta))_{N\geq0}$ is all-time calibrated for $f^*$ if
\[
\mathbb P\!\left(f^*\in\bigcap_{N\geq0}\mathcal M_N(\delta)\right)\geq1-\delta.
\]
\end{definition}

\paragraph{Coordinate-wise GP posterior.}
\label{sec:gp_posterior}
We equip each coordinate estimator with an independent prior $f_j\sim\mathcal{GP}(0,k)$ sharing a scalar kernel $k$. Given $z_i=(s_i,a_i)$ and $y_i=f^*(z_i)+\epsilon_i$, where $\epsilon_i\sim\mathcal N(0,\sigma^2I_{n_S})$, define
\[
\begin{aligned}
K_N&=[k(z_i,z_r)]_{i,r=1}^N,\\
k_N(z)&=[k(z,z_1),\ldots,k(z,z_N)]^\top.
\end{aligned}
\]
Writing $y_{1:N}^{(j)}=[y_1^{(j)},\ldots,y_N^{(j)}]^\top$, standard GP conditioning gives
\begin{subequations}
\label{eq:gp_scalar_posterior}
\begin{align}
\mu_{N,j}(z)
&=k_N(z)^\top(K_N+\sigma^2I_N)^{-1}y_{1:N}^{(j)},\\
\sigma_{N,j}^2(z)
&=k(z,z)-k_N(z)^\top(K_N+\sigma^2I_N)^{-1}k_N(z).
\end{align}
\end{subequations}

\begin{assumption}[RKHS Regularity]
\label{assumption:rkhs_regularity_assumption}
For every $j\in[n_S]$, the coordinate $f_j^*:Z\to\mathbb R$ belongs to the RKHS $\mathcal H_k$ induced by $k$ and satisfies
\(
\|f_j^*\|_{\mathcal H_k}\leq B.
\)
The kernel is uniformly bounded on the diagonal:
\(
\|k\|_\infty\coloneqq\sup_{z\in Z}k(z,z)\leq\kappa.
\)
\end{assumption}

The resulting coordinate-wise posterior mean and diagonal covariance are denoted by
\begin{subequations}
\label{eq:post_mean_cov}
\begin{align}
\label{eq:post_mean}
\mu_N(z)&=\bigl(\mu_{N,1}(z),\ldots,\mu_{N,n_S}(z)\bigr)^\top,\\
\label{eq:post_cov}
\Sigma_N(z)&=\operatorname{diag}\!\left(\sigma_{N,1}^2(z),\ldots,\sigma_{N,n_S}^2(z)\right),
\end{align}
\end{subequations}
with $\sigma_{N,j}^2(z)\leq\kappa$.

\paragraph{Bounded deterministic component and unbounded state evolution.}
Assumption~\ref{assumption:rkhs_regularity_assumption} implies that the deterministic transition component is uniformly bounded. By the reproducing property and Cauchy--Schwarz inequality,
\[
|f_j^*(z)|
=
|\langle f_j^*,k(\cdot,z)\rangle_{\mathcal H_k}|
\leq
\|f_j^*\|_{\mathcal H_k}\sqrt{k(z,z)}
\leq B\sqrt\kappa.
\]
Consequently,
\[
\|f^*(z)\|_2\leq\sqrt{n_S}\,B\sqrt\kappa,
\qquad z\in Z.
\]
Thus, the deterministic component is controlled uniformly over the unbounded state space, while the Gaussian perturbation in~\cref{eq:state_transition_definition} has unbounded support and permits genuinely unbounded state evolution. Uniformly bounded kernels and the resulting control of RKHS transition functions are standard in nonlinear-system analysis~\cite{mania2020active,kakade2020information,curi2020efficient,wagenmaker2023optimal}.

For $U=(u_1,\ldots,u_N)\in Z^N$, let $K_U=[k(u_r,u_s)]_{r,s=1}^N$. The maximum information gain used below is
\[
\Gamma_N
\coloneqq
\sup_{U\in Z^N}
\frac{1}{2}\log\det\!\left(I_N+\sigma^{-2}K_U\right),
\]
the standard kernel-dependent complexity of $N$ noisy observations~\cite{srinivas2009gaussian,chowdhury2017kernelized}.

\begin{lemma}[Well-Calibrated Confidence Sets~\cite{sukhija2024neorl}]
\label{lemma:well_calibrated_gp}
Fix $\delta\in(0,1)$ and suppose Assumption~\ref{assumption:rkhs_regularity_assumption} holds. There exists a sequence $(\beta_N(\delta))_{N=0}^{\infty}$ proportional to $\sqrt{\Gamma_N}$, where $\Gamma_N$ is the maximum information gain of $k$ after $N$ observations, such that the confidence sequence induced by $(\mu_N,\sigma_N,\beta_N(\delta))_{N=0}^{\infty}$ is all-time calibrated for $f^*$.
\end{lemma}

This lemma converts posterior uncertainty into a simultaneous high-probability event and supplies the statistical quantity that appears explicitly in our main convergence theorem.

\subsection{Finite-game approximation}

We conclude the formulation with the approximation result that connects the continuous game to finite equilibrium computation. We first recall the discretization device.

\begin{definition}[$\delta_G$-net]
\label{def:delta_net}
Let $\delta_G>0$ and let $(E,d)$ be a compact metric space. A finite subset $E_{\delta_G}\subseteq E$ is a $\delta_G$-net of $E$ if every $e\in E$ satisfies $d(e,e_i)\leq\delta_G$ for some $e_i\in E_{\delta_G}$.
\end{definition}

Following~\citet{saldi2024existence}, choose an increasing sequence of compact sets $(\mathcal C_N)_{N\geq0}$ satisfying
\[
S=\bigcup_{N\geq0}\mathcal C_N,
\qquad
\mathcal C_N\subset\operatorname{int}(\mathcal C_{N+1}).
\]
For example, one may take
\begin{equation}
\label{eq:K_n}
\mathcal C_N=[-r_N,r_N]^{n_S},
\qquad r_N\uparrow\infty.
\end{equation}
Let $S_{\delta_G}^N$ and $A_{i,\delta_G}$ be $\delta_G$-nets of $\mathcal C_N$ and $A_i$, respectively, and choose a pseudo-state $\Delta_N\in\mathcal C_N^c$. The finite state space is
\[
X_N\coloneqq S_{\delta_G}^N\cup\{\Delta_N\}.
\]
Define the state and action projections by
\[
D_{\delta_G}^N(x)
=
\begin{cases}
\displaystyle\arg\min_{z\in S_{\delta_G}^N}\|x-z\|,&x\in\mathcal C_N,\\[0.6em]
\Delta_N,&x\in\mathcal C_N^c,
\end{cases}
\]
and $D_{i,\delta_G}(a^i)=\arg\min_{b\in A_{i,\delta_G}}\|a^i-b\|$. Resolving non-unique minimizers measurably~\cite{wagner1977survey}, the cells
\[
S_{i,\delta_G}^N
\coloneqq
\{x\in\mathcal C_N:D_{\delta_G}^N(x)=x_i\},
\qquad x_i\in S_{\delta_G}^N,
\]
give the finite measurable partition
\[
S=\mathcal C_N^c\cup\bigcup_{x_i\in S_{\delta_G}^N}S_{i,\delta_G}^N.
\]
Let $p$ denote the Gaussian transition density of the original game and let $(\nu_N)_{N\geq0}$ be probability measures supported on $\mathcal C_N^c$. Define the aggregated transition density and rewards by
\[
p_N(x\mid s,a)
=
\begin{cases}
p(x\mid s,a),&s\in\mathcal C_N,\\
\displaystyle\int_{\mathcal C_N^c}p(x\mid z,a)\,\nu_N(dz),&s\in\mathcal C_N^c,
\end{cases}
\]
and
\[
\mathcal R_{i,N}(s,a)
=
\begin{cases}
\mathcal R_i(s,a),&s\in\mathcal C_N,\\
\displaystyle\int_{\mathcal C_N^c}\mathcal R_i(z,a)\,\nu_N(dz),&s\in\mathcal C_N^c.
\end{cases}
\]
Writing $\lambda$ for Lebesgue measure, the transition probability mass function on $X_N$ is
\[
h_N(x\mid s,a)
\coloneqq
\begin{cases}
\displaystyle\int_{S_{i,\delta_G}^N}p_N(z\mid s,a)\,\lambda(dz),
&x=x_i,\\[0.8em]
\displaystyle\int_{\mathcal C_N^c}p_N(z\mid s,a)\,\lambda(dz),
&x=\Delta_N.
\end{cases}
\]
Thus the finite game is
\[
G_N(\delta_G)
\coloneqq
\bigl(X_N,(A_{i,\delta_G})_{i=1}^n,(\mathcal R_{i,N})_{i=1}^n,h_N\bigr).
\]
This construction makes $G_N(\delta_G)$ a well-defined finite stochastic game while the increasing sequence $(\mathcal C_N)$ expands toward the full state space.

\begin{proposition}[Finite-Game Approximation~\cite{saldi2024existence}]
\label{prop:approximation}
Fix an arbitrary compact set $\mathcal C\subset S=\mathbb R^{n_S}$ and let $\varepsilon>0$. There exist a discretization resolution $\delta_G>0$ and a positive integer $M=M(\mathcal C,\varepsilon,\delta_G)$ such that, for every $m\geq M$, a Nash equilibrium $\pi^m$ of $G_m(\delta_G)$ induces an $\varepsilon$-Nash equilibrium of the original game $G$ on $\mathcal C$; that is, for every $i\in[n]$, $s\in\mathcal C$, and admissible deviation $\pi_i$,
\[
V_i^{\rho^m}(s)
\geq
V_i^{(\pi_i,\rho_{-i}^m)}(s)-\varepsilon,
\]
where $\rho^m$ is the extension of $\pi^m$ to the original state space (see Remark~\ref{remark:policy_extension}).
\end{proposition}

\emph{The set $\mathcal C$ is not an operating-domain restriction: agents interact on $S$, and the values above integrate complete trajectories on $S$}. It is an arbitrary set of states on which the finite-approximation guarantee is uniform. Proposition~\ref{prop:approximation} supplies this deterministic component to the convergence analysis below.

\section{Algorithm and Theoretical Analysis}
\label{sec:algorithm_and_theoretical_analysis}
GP-MARL follows centralized training with decentralized execution (CTDE)~\cite{lowe2017multi}. A central learner fits the transition model from joint interaction data, solves a finite approximation of the learned game, and distributes the resulting policies for decentralized execution. We use this framework to characterize high-probability sufficient conditions for convergence to an $\varepsilon$-Nash equilibrium.

\begin{algorithm}[t]
\caption{GP-MARL}
\label{alg:finite_game_rl}
\begin{algorithmic}[1]
\Require Kernel $k$, Nash equilibrium error $\varepsilon > 0$, discount factor $\gamma\in[0,1)$, mesh size $\delta_G>0$, finite state approximations $(X_t)_{t=0}^N$, and finite action approximations
$A_{i,\delta_G}$ for $i \in [n]$.

\State Initialize policies $\pi_1^0,\ldots,\pi_n^0$
\State Initialize GP model with prior mean $\mu_0$, variance matrix $\Sigma_0$, and true variance estimate $\Hat{\sigma}^2 \ge 0$
\State Initialize replay buffer $\mathcal D \gets \emptyset$
\State Sample initial state $s_0 \sim \mu$ (see Definition~\ref{def:general_sum_stochastic_game})

\For{$t=0,1,\ldots,N$}
    \State Update the coordinate-wise GP posterior mean $\Hat{f}_t=\mu_t$ (Equation~\ref{eq:post_mean}) and covariance $\Sigma_{\Hat{f}_t}=\Sigma_t$ (Equation~\ref{eq:post_cov}) using replay buffer $\mathcal{D}$; see Section~\ref{sec:gp_posterior}
    \State Update estimated transition density $\widehat{\mathcal{T}}_t(\cdot \mid s, a) = \mathcal{N}(\cdot \mid \Hat{f}_t(s, a), \Sigma_{\Hat{f}_t}(s, a) + \Hat{\sigma}^2 I_{n_S})$
    \State Construct the finite game $\widehat G_t(\delta_G)$ on $X_t$
    and $A_{1,\delta_G},\ldots,A_{n,\delta_G}$ using
    $\widehat{\mathcal T}_{t}$
    \State Compute a $\varepsilon$-Nash equilibrium
    $\pi^t=(\pi_1^t,\ldots,\pi_n^t)$ of $\hat G_t(\delta_G)$
    \State Extend $\pi^t$ to the original state-action space according to
    Remark~\ref{remark:policy_extension} and distribute the individual
    policies to the agents
    \State Agents independently sample actions $a_t \sim \pi^t(\cdot \mid s_t)$
    \State Observe the next state $s_{t+1} \sim \mathcal{T}(\cdot \mid s_t, a_t)$
    \State Store $(s_t, a_t, s_{t+1})$ in replay buffer $\mathcal{D}$
\EndFor

\State \Return $\pi^N=(\pi_1^N,\ldots,\pi_n^N)$
\end{algorithmic}
\end{algorithm}

The convergence analysis uses complementary value-function stability bounds for perturbations of the transition model and policy.
\begin{lemma}
\label{lemma:lemma1}
    Suppose Assumption~\ref{assumption:structural_assumptions} holds. Let $f,g:S\times A\to\bb{R}^{n_S}$ be continuous and bounded, and let $\Sigma_f,\Sigma_g:S\times A\to\bb{R}^{n_S\times n_S}$ be symmetric positive-definite matrix-valued maps satisfying $\sup_{(s,a)}\max\{\|\Sigma_f(s,a)\|_2^2,\|\Sigma_g(s,a)\|_2^2\}<\infty$. Define $\mathcal{T}_h(\cdot\mid s,a)=\mathcal{N}(h(s,a),\Sigma_h(s,a))$ for $h\in\{f,g\}$, and denote the corresponding value functions by $V_{i,h}^\pi$. Write $D_{f,g}(s,a)\coloneqq D_{\mathrm{KL}}(\mathcal T_f(\cdot\mid s,a)\|\mathcal T_g(\cdot\mid s,a))$. If $\gamma\in[0,1)$, then
    \[
    \begin{aligned}
        \|V_{i,f}^\pi-V_{i,g}^\pi\|_\infty
        &\le
        \frac{\sqrt{2}\,\gamma M_i}{(1-\gamma)^2}\\
        &\quad\times\sup_{(s,a)\in S\times A}
        \sqrt{D_{f,g}(s,a)}.
    \end{aligned}
    \]
\end{lemma}
\noindent\emph{Proof.} See Appendix~\ref{app:proof_transition_stability}.

\begin{lemma}
\label{lemma:lemma2}
Suppose Assumption~\ref{assumption:structural_assumptions} holds, so $|\mathcal{R}_i(s,a)|\le M_i$ for every $(s,a)\in S\times A$ and $i\in[n]$. For valid policies $\pi=(\pi_1,\ldots,\pi_n)$ and $\rho=(\rho_1,\ldots,\rho_n)$, let $d_{\mathrm{TV}}(\pi,\rho)\coloneqq\sup_{s\in S}d_{\mathrm{TV}}(\pi(\cdot\mid s),\rho(\cdot\mid s))$, where $d_{\mathrm{TV}}(\mu,\nu)\coloneqq\frac{1}{2}\int_A|\mu-\nu|(da)$. Then, for every $i\in[n]$,
\[
    \left\|V_i^\pi-V_i^\rho\right\|_\infty
    \le
    \frac{2M_i}{(1-\gamma)^2}
    d_{\mathrm{TV}}(\pi,\rho).
\]
\end{lemma}

\noindent\emph{Proof.} See Appendix~\ref{app:proof_policy_stability}.

\begin{remark}[Localized Stability Bounds]
\label{remark:localized_stability_bounds}
For a compact set $\mathcal{C}\subset S$, write
\[
\begin{aligned}
    \|h\|_{\mathcal{C}}
    &\coloneqq
    \sup_{s\in\mathcal{C}}|h(s)|,\\
    d_{\mathrm{TV},\mathcal{C}}(\pi,\rho)
    &\coloneqq
    \sup_{s\in\mathcal{C}}
    d_{\mathrm{TV}}\!\left(\pi(\cdot\mid s),\rho(\cdot\mid s)\right).
\end{aligned}
\]
Lemmas~\ref{lemma:lemma1}--\ref{lemma:lemma2} remain valid in the localized form used with Proposition~\ref{prop:approximation}. Here $\|\cdot\|_{\mathcal C}$ evaluates full-state-space value functions at initial states in $\mathcal C$: Bellman integrals and trajectories remain on $S$. The policy comparison is local, so the selected policies need agree on $\mathcal C$ only and may vary outside it.
\end{remark}

\begin{remark}
\label{remark:policy_extension}
We may extend the joint policies $(\pi^k)_{k=0}^\infty$ to a joint policy on $S \times A$ through the following procedure as in~\cite{saldi2024existence}:
\[
\frac{d\rho^k}{d\lambda}(\cdot\mid s)
\coloneqq
\sum_{a\in A_{\delta_G}}
\pi^k\!\left(a\mid D_{\delta_G}^k(s)\right)\delta_a(\cdot),
\qquad s\in S,
\]
where
\[
A_{\delta_G}
=
\bigtimes_{i=1}^n A_{i,\delta_G}
\]
and $\delta_a$ is the Dirac delta function centered at $a$.
\end{remark}

The result below characterizes convergence: any policy-selection and data-collection procedure satisfying its stated conditions inherits the guarantee. Fix $\varepsilon>0$ and an arbitrary compact $\mathcal C\subset S$, using $d_{\mathrm{TV},\mathcal C}$ from Remark~\ref{remark:localized_stability_bounds}. At iteration $\ell$, let the true kernel $\mathcal T$ and learned Gaussian kernel $\mathcal T_\ell$ satisfy the conditions of Lemma~\ref{lemma:lemma1}, and define
\[
\eta_\ell
\coloneqq
\sup_{(s,a)\in\mathcal{C}\times A}
\sqrt{D_{\mathrm{KL}}\!\left(
\mathcal{T}(\cdot\mid s,a)
\,\middle\|\,
\mathcal{T}_\ell(\cdot\mid s,a)
\right)}.
\]
For the GP specialization, let $\mathcal{D}_\ell$ contain the transition
observations collected so far and let $\mu_\ell$ and $\Sigma_\ell$ be the
corresponding posterior mean and covariance, so that
\[
\begin{aligned}
\mathcal{T}_\ell(\cdot\mid s,a)
&=
\mathcal{N}\!\left(
\mu_\ell(s,a),
\Sigma_\ell(s,a)+\sigma^2 I_{n_S}
\right),\\
\mathcal{T}(\cdot\mid s,a)
&=
\mathcal{N}\!\left(
f^*(s,a),
\sigma^2 I_{n_S}
\right),
\end{aligned}
\]
respectively. Define
\[
q_\ell
\coloneqq
\sup_{\substack{(s,a)\in\mathcal{C}\times A\\j\in[n_S]}}
\sigma_{\ell,j}^2(s,a),
\qquad
n_S\coloneqq\sum_{i=1}^n d_i.
\]
The bounded-kernel RKHS discussion in Section~\ref{sec:problem_formulation} gives uniform bounds on $f^*$ and the posterior variances, so the GP specialization satisfies the required boundedness conditions and $q_\ell<\infty$ for every $\ell$.

Consider the finite-game approximations $(G_\ell(\delta))_{\ell=0}^\infty$
from \S~\ref{sec:problem_formulation}, where
$S_\delta^\ell\subseteq S_\delta^{\ell+1}$ and $S_\delta^\ell$ is a
$\delta$-net of
\[
\mathcal{C}_\ell\coloneqq[-r_\ell,r_\ell]^{n_S},
\]
for an increasing sequence $(r_\ell)_{\ell = 0}^\infty$ satisfying
$r_\ell\to+\infty$. Let $\delta\leq\delta_G$, where $\delta_G$ is given
by Proposition~\ref{prop:approximation}, and let
$(\rho^\ell)_{\ell=0}^\infty$ be the resulting joint policies, extended
to the original state-action space as in
Remark~\ref{remark:policy_extension}. For a policy $\rho^*$ whose restriction to $\mathcal C$ is the limit of the selected policies, define
\[
\Delta_\ell^\rho
\coloneqq
d_{\mathrm{TV},\mathcal{C}}(\rho^\ell,\rho^*).
\]

\begin{theorem}[KL-to-Value Error and High-Probability Convergence]
\label{theorem:convergence_rate}
Suppose Assumptions~\ref{assumption:structural_assumptions}--%
\ref{assumption:f_is_continuous} hold. For any sequence of learned
Gaussian transition kernels $(\mathcal T_\ell)_{\ell\geq0}$ described
above and every $i\in[n]$,
\[
\left\|
V_i^{\rho^\ell}(\cdot;\mathcal{T}_\ell)
-
V_i^{\rho^*}(\cdot;\mathcal{T})
\right\|_{\mathcal{C}}
\leq
\frac{2M_i}{(1-\gamma)^2}\Delta_\ell^\rho
+
\frac{\sqrt{2}\gamma M_i}{(1-\gamma)^2}\eta_\ell.
\]
Hence, stable equilibrium selection on $\mathcal C$ ($\Delta_\ell^\rho\to0$) and
uniform KL convergence ($\eta_\ell\to0$) imply value convergence on
$\mathcal C$ and, via Proposition~\ref{prop:approximation}, ensure that
$\rho^*$ is an $\varepsilon$-Nash equilibrium on $\mathcal C$ in the sense of
Proposition~\ref{prop:approximation} for any such Gaussian transition 
estimator. If, additionally, Assumption~\ref{assumption:rkhs_regularity_assumption}
holds and the coordinates of $f^*$ are modeled using Algorithm~\ref{alg:finite_game_rl}, then, for every
$\alpha\in(0,1)$, with probability at least $1-\alpha$, simultaneously
for all $\ell\geq0$,
\[
\eta_\ell^2
\leq
n_S\!\left(
\frac{q_\ell^2}{4\sigma^4}
+
\frac{\beta_\ell^2(\alpha)q_\ell}{2\sigma^2}
\right),
\]
and hence, using
$\beta_\ell(\alpha)=\mathcal{O}(\sqrt{\Gamma_\ell})$,
\[
\begin{aligned}
&\left\|
V_i^{\rho^\ell}(\cdot;\mathcal{T}_\ell)
-
V_i^{\rho^*}(\cdot;\mathcal{T})
\right\|_{\mathcal{C}}\\
&\quad=
\mathcal{O}\!\left(
\Delta_\ell^\rho
+
\sqrt{n_S}
\left[
\frac{q_\ell}{\sigma^2}
+
\frac{\sqrt{\Gamma_\ell q_\ell}}{\sigma}
\right]
\right).
\end{aligned}
\]
Consequently, if $\Delta_\ell^\rho\to0$, $q_\ell\to0$, and
$\Gamma_\ell q_\ell\to0$, then,
with probability at least $1-\alpha$,
\[
\left\|
V_i^{\rho^\ell}(\cdot;\mathcal{T}_\ell)
-
V_i^{\rho^*}(\cdot;\mathcal{T})
\right\|_{\mathcal{C}}
\to0,
\]
and $\rho^*$ is an $\varepsilon$-Nash equilibrium on $\mathcal C$ in the sense of
Proposition~\ref{prop:approximation}.
\end{theorem}

\noindent\emph{Proof.} See Appendix~\ref{app:proof_convergence_rate}.

\paragraph{Proof structure.}
The estimator-agnostic statement follows from
\[
\begin{aligned}
&\left\|V_i^{\rho^\ell}(\cdot;\mathcal T_\ell)
-V_i^{\rho^*}(\cdot;\mathcal T)\right\|_{\mathcal C}\\
&\quad\leq
\left\|V_i^{\rho^\ell}(\cdot;\mathcal T_\ell)
-V_i^{\rho^\ell}(\cdot;\mathcal T)\right\|_{\mathcal C}\\
&\qquad+
\left\|V_i^{\rho^\ell}-V_i^{\rho^*}\right\|_{\mathcal C},
\end{aligned}
\]
followed by Lemmas~\ref{lemma:lemma1}--\ref{lemma:lemma2}. For the GP specialization, fix $z=(s,a)$ and write $r_{\ell j}=\sigma_{\ell,j}^2(z)/\sigma^2$. The Gaussian KL divergence is exactly
\[
\begin{aligned}
&D_{\mathrm{KL}}\!\left(\mathcal T(\cdot\mid z)\,\middle\|\,
\mathcal T_\ell(\cdot\mid z)\right)\\
&\quad=\frac12\sum_{j=1}^{n_S}\Biggl[
\log(1+r_{\ell j})-\frac{r_{\ell j}}{1+r_{\ell j}}\\
&\hspace{7.5em}+
\frac{(\mu_{\ell,j}(z)-f_j^*(z))^2}
{\sigma^2+\sigma_{\ell,j}^2(z)}\Biggr].
\end{aligned}
\]
On the all-time confidence event, the last term is at most
$\beta_\ell^2(\alpha)r_{\ell j}$. Applying
$\log(1+r)-r/(1+r)\leq r^2/2$ and $r_{\ell j}\leq q_\ell/\sigma^2$
gives the stated bound on $\eta_\ell^2$. Appendix~\ref{app:proof_convergence_rate} provides the complete argument, including the $\varepsilon$-Nash conclusion.

\paragraph{Roles of the assumptions.}
Assumptions~\ref{assumption:structural_assumptions}--\ref{assumption:f_is_continuous}, together with the Gaussian-kernel conditions stated above, support the estimator-agnostic part of Theorem~\ref{theorem:convergence_rate} and the finite-game conclusion. Assumption~\ref{assumption:rkhs_regularity_assumption} and the all-time calibration result are used for the GP-specific high-probability bound. In both cases, $\mathcal C$ specifies the initial states over which the error is uniform; transition integrals and trajectories remain on $S$.

\begin{remark}[Interpretation of the Convergence Bound]
\label{remark:policy_convergence}
\label{remark:dependence_on_the_number_of_agents}
\label{remark:discussion_of_gp-marl}
The rate separates
\[
\begin{aligned}
&\underbrace{\Delta_\ell^\rho}_{\text{equilibrium selection}}
+
\underbrace{\sqrt{n_S}
\left[
\frac{q_\ell}{\sigma^2}
+
\frac{\sqrt{\Gamma_\ell q_\ell}}{\sigma}
\right]}_{\text{transition estimation}},\\
&n_S=\sum_{i=1}^n d_i.
\end{aligned}
\]
The term $\Delta_\ell^\rho=d_{\mathrm{TV},\mathcal C}(\rho^\ell,\rho^*)$ measures agreement only on the arbitrary set $\mathcal C$; the selected policies may change outside it. The theorem is deliberate about how this local stability and statistical contraction are achieved. The quantity $q_\ell$ is the largest posterior variance on $\mathcal C\times A$, while $\Gamma_\ell$ controls confidence width. Thus, $\sqrt{n_S}$ is the direct output-dimension factor, with further dependence through $q_\ell$ and $\Gamma_\ell$. More generally, $\Delta_\ell^\rho\to0$ and $\eta_\ell\to0$ yield value convergence and an $\varepsilon$-Nash equilibrium on $\mathcal C$; the GP conditions provide an explicit high-probability guarantee for $\eta_\ell\to0$.
\end{remark}

\section{Controlled Numerical Illustration}
\label{sec:experiments}

The paper's contribution is theoretical; this controlled study visualizes the kernel and dimensional quantities in Theorem~\ref{theorem:convergence_rate}. Finite games are solved accurately to isolate transition learning, with implementation details in Appendix~\ref{app:experimental_details} (code is available \href{https://anonymous.4open.science/r/gpmarl-443B}{here}).

\paragraph{Benchmark.}
Each of the $n$ agents has a scalar state and action, so $n_S=n$, $S=\mathbb{R}^n$, and $A_i=[-1,1]$. The dynamics and common reward are
\[
\begin{aligned}
s_{t+1}&=f^*(s_t,a_t)+w_t,\\
f_j^*(s_t,a_t)&=\tanh(s_{t,j}+0.5a_{t,j}),\\
w_t&\sim\mathcal{N}(0,0.1^2I_{n_S}),
\end{aligned}
\]
and
\[
\mathcal{R}_i(s,a)=-\min\{\|s\|_2^2+\|a\|_2^2,4\}.
\]
These dynamics and rewards satisfy the paper's boundedness conditions. State transitions are modeled coordinate-wise with a common RBF or Mat\'ern kernel satisfying Assumption~\ref{assumption:rkhs_regularity_assumption}.

\paragraph{Finite-game implementation.}
At every recorded update $\ell$, we use a fixed compact set and state net, denoted by $\mathcal C_\ell=\mathcal C$ and $S_{\delta_G}^\ell=S_{\delta_G}$, and choose a fixed pseudo-state $\Delta_\ell = \Delta \in\mathcal C^c$. Taking $\nu_\ell$ to be the Dirac measure at $\Delta_\ell$ gives
\[
p_\ell(x\mid s,a)
=
\begin{cases}
p(x\mid s,a),&s\in\mathcal C,\\
p(x\mid\Delta_\ell,a),&s\in\mathcal C^c.
\end{cases}
\]
Taking $\widehat\sigma^2 = 0$, Algorithm~\ref{alg:finite_game_rl} forms
\[
\widehat{\mathcal T}_\ell(\cdot\mid s,a)
=
\mathcal N\!\left(
\cdot\mid\mu_\ell(s,a),
\Sigma_\ell(s,a)
\right),
\]
constructs $\widehat G_\ell(\delta_G)$ from this learned kernel, solves it for an $\varepsilon$-Nash equilibrium, and extends the resulting policy as in Remark~\ref{remark:policy_extension}. Unless stated otherwise, GP-MARL is initialized with a uniform policy.

\paragraph{Protocol and metric.}
At each update, GP-MARL constructs and solves the finite game induced by its GP posterior. The benchmark admits a high-accuracy reference policy obtained by one-dimensional value iteration with $\gamma=0.9$; the true dynamics are used only for evaluation, while GP-MARL learns exclusively from observed transitions. We report $\ln e_\ell$, where
\[
e_\ell
\coloneqq
\left|
V_1^{\rho^\ell}(s_{t_\ell};\widehat{\mathcal{T}}_\ell)
-
V_1^*(s_{t_\ell};\mathcal{T})
\right|.
\]
This metric is a trajectory-level counterpart of the theorem's value discrepancy. Appendix~\ref{app:experimental_details} specifies the grids, update schedules, kernel calibration, reference computation, and quadrature.

The evaluation reference uses knowledge of $f^*$, whereas GP-MARL constructs its game from $(\mu_\ell,\Sigma_\ell)$ using only transitions in $\mathcal D_\ell$. Thus, $\rho^\ell$ reflects both the current transition estimate and its posterior uncertainty over $S\times A$. The learned transition model is defined on the full joint input $z=(s,a)$, and the resulting finite game is solved in its multi-agent form. We evaluate the learned policy through the pointwise value discrepancy at states encountered along the trajectory, providing a trajectory-level illustration of value convergence.

\begin{figure}[t]
    \centering
    \includegraphics[width=\linewidth]{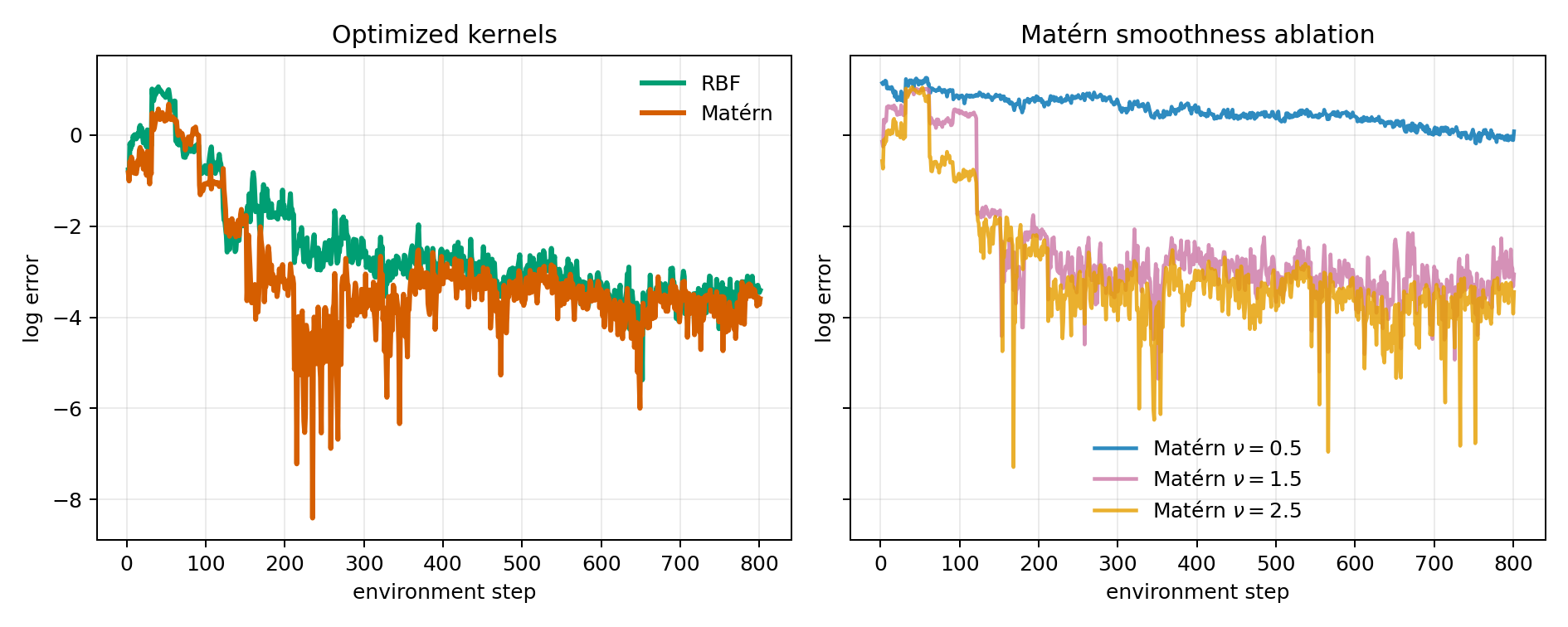}
    \caption{Kernel analysis for two-player GP-MARL. Left: value-error convergence for the RBF and Mat\'ern kernels. Right: effect of the Mat\'ern smoothness parameter, with all other kernel hyperparameters fixed.}
    \label{fig:optimized_2player}
\end{figure}

\begin{figure}[t]
    \centering
    \includegraphics[width=\linewidth]{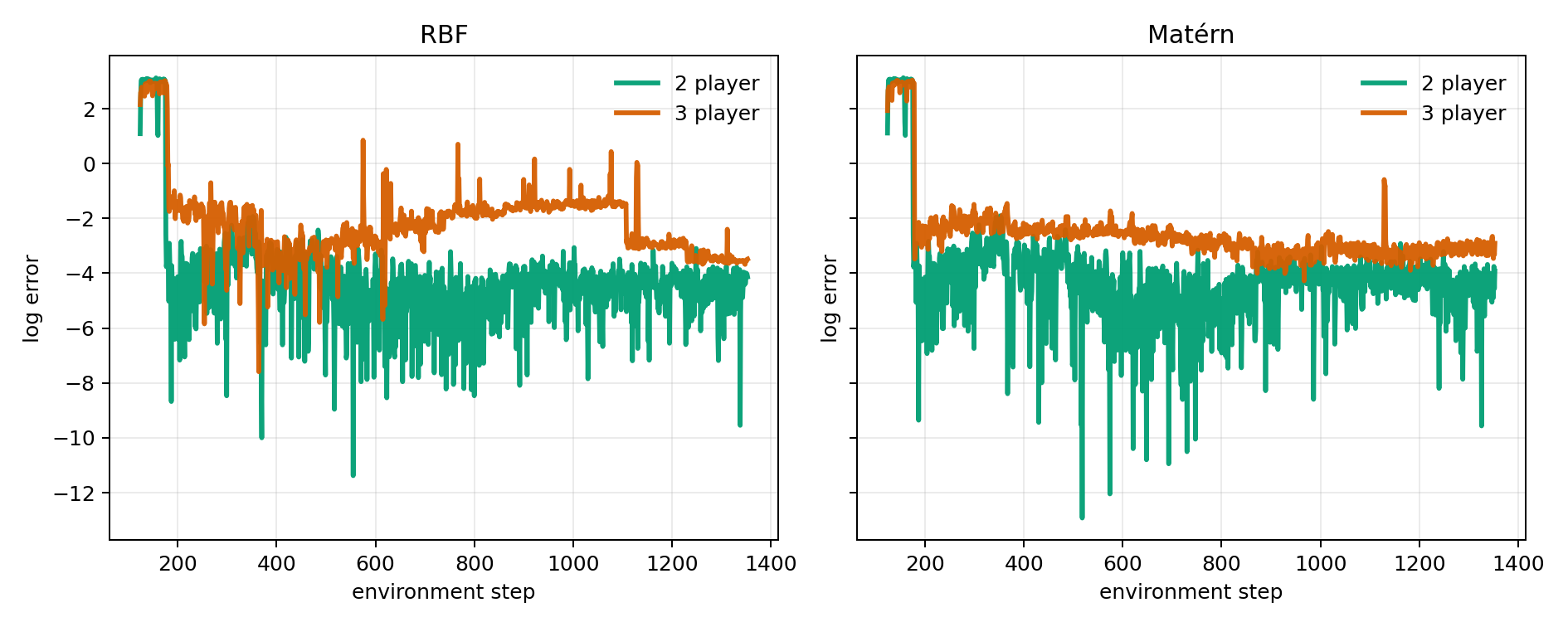}
    \caption{Agent-scaling analysis for GP-MARL. Left: two- and three-player value-error convergence with the RBF kernel. Right: the corresponding comparison with the Mat\'ern kernel.}
    \label{fig:varying_players}
\end{figure}

\paragraph{Results.}
Figure~\ref{fig:optimized_2player} shows that the optimized Mat\'ern and RBF kernels first attain $e_\ell<0.05$ after $151$ and $312$ interactions, respectively. The smoothness ablation illustrates the kernel-dependent roles of Theorem~\ref{theorem:convergence_rate}.

The agent-scaling study in Figure~\ref{fig:varying_players} reaches $e_\ell<0.1$ after $175$ and $252$ interactions for two and three agents with the RBF kernel, respectively, and after $175$ and $178$ interactions with the Mat\'ern kernel, illustrating the role of the state-space dimension $n_S$.

\section{Discussion}
\label{sec:discussion}
The controlled illustration connects the observed empirical behavior to two dependencies made explicit in Theorem~\ref{theorem:convergence_rate}. First, the kernel comparison illustrates how kernel regularity influences convergence, with greater Mat\'ern smoothness benefiting this infinitely differentiable benchmark. Second, the two- and three-agent comparison illustrates the role of the state-space dimension $n_S$ in the convergence behavior. Together, these results provide a focused empirical interpretation of the theorem's kernel- and dimension-dependent terms.

\section{Limitations and Future Work}
\label{sec:limitations_and_future_work}

Theorem~\ref{theorem:convergence_rate} provides algorithm-independent design targets by separating local policy-selection stability from statistical contraction. Establishing selection and data-acquisition guarantees that imply $\Delta_\ell^\rho\to0$ and $\eta_\ell\to0$ is a natural next theoretical direction; for GP models, $q_\ell\to0$ and $\Gamma_\ell q_\ell\to0$ give an explicit objective. Promising extensions include estimating the transition-noise variance $\widehat{\sigma}^2$ from the interaction trajectory and extending the analysis to broader information structures, including partial observability.

\section{Conclusion}
\label{sec:conclusion}

We introduced GP-MARL, a CTDE framework for analyzing continuous general-sum stochastic games with unknown Gaussian dynamics. On any prescribed compact set of states, with interaction defined on the full unbounded state space, local policy-selection stability and uniform KL convergence guarantee value convergence and an $\varepsilon$-Nash equilibrium in the sense of Proposition~\ref{prop:approximation}. The GP specialization verifies the statistical condition through posterior variance and kernel information gain, with explicit dependence on joint state dimension. Controlled numerical results illustrate the value-convergence behavior captured by these theoretical guarantees.

\appendix
\section{Proofs}
\label{app:proofs}

\subsection{Proof of Proposition~\ref{prop:structural_assumptions_satisfied}}
\label{app:proof_structural_assumptions}
\begin{proof}
It remains only to verify that Assumption~\ref{assumption:structural_assumptions}(e) is satisfied. Define
\[
X = f^*(s, a) + w,
\]
where $w \sim \mathcal{N}(0, \sigma^2I_{n_S})$. Since $f^*(s, a) \in S$ is deterministic, it follows immediately that $X$ is multivariate normal with mean $f^*(s, a)$ and covariance matrix $\sigma^2I_{n_S}$. 

Therefore, the Radon--Nikodym derivative
\[
\frac{d\mathcal{T}(\cdot \mid s, a)}{d\lambda}
\]
is a multivariate Gaussian density, where $\lambda$ denotes Lebesgue measure on $\bb{R}^{n_S}$.

We now establish the existence of a probability measure $\bb{P}$ and a function $b \colon S \times A \times S \to \bb{R}_+$ satisfying
\[
    \int_S \linfty{b(s, \cdot, s^\prime)} \bb{P}(ds^\prime) < +\infty,
\]
with $b(s, \cdot, \cdot)$ a Carathéodory function. Since the Borel $\sigma$-algebra $\mathcal{S}$ on $S$ is countably generated, we may define
\[
    h(s, a, \cdot) \coloneqq \frac{d\mathcal{T}(\cdot \mid s, a)}{d\lambda},
\]
and modify it appropriately to obtain the desired function $b$. Observe that although $h$ is defined on $S \times A \times S$, the Lebesgue measure $\lambda$ is not a probability measure on $(S, \mathcal{S})$. We therefore proceed as follows.

Let $\rho \colon S \to \bb{R}_+$ be a strictly positive continuous density satisfying
\[
\int_S \rho(s) \lambda(ds)  = 1,
\]
and define
\[
\bb{P}(B) \coloneqq \int_B \rho(s) \lambda(ds), \qquad B\in\mathcal S.
\]
Then $\bb{P}$ is a probability measure~\cite{cohn2013measure}. We now set
\[
    b(s, a, \cdot) = \frac{h(s, a, \cdot)}{\rho(\cdot)}.
\]
By construction,
\[
    \frac{d\bb{P}}{d\lambda} = \rho,
\]
and therefore
\begin{equation*}
\begin{split}
    \int_S b(s, a, s^\prime) \bb{P}(ds^\prime) & = \int_S b(s, a, s^\prime) \rho(s^\prime) \lambda(ds^\prime) \\
    & = \int_S h(s, a, s^\prime) \lambda(ds^\prime) \\
    & = \int_S \mathcal{T}(ds^\prime \mid s, a). \\
\end{split}
\end{equation*}

Moreover, we obtain the estimate
\[
\begin{split}
    \linfty{b(s, \cdot, s^\prime)} & = \linfty{\frac{h(s, \cdot, s^\prime)}{\rho(s^\prime)}} \\
    & \le \frac{g_s(s^\prime)}{\rho(s^\prime)}, \quad \text{(Lemma~\ref{lemma:existence_of_bound}).} \\
\end{split}
\]
Consequently,
\[
\begin{split}
    \int_S \linfty{b(s, \cdot, s^\prime)} \bb{P}(ds^\prime) & \le \int_S \frac{g_s(s^\prime)}{\rho(s^\prime)} \rho(s^\prime) \lambda(ds^\prime) \\
    & = \int_S g_s(s^\prime) \lambda(ds^\prime) \\
    & < +\infty. \\
\end{split}
\]

Finally, continuity of $f^*$ and strict positivity and continuity of $\rho$ imply that $b(s,\cdot,\cdot)$ is continuous in $(a,s')$ for every fixed $s$. Its explicit Gaussian form also makes it product-measurable, so it is a Carath\'eodory function.
\end{proof}

\subsection{Proof of Lemma~\ref{lemma:lemma1}}
\label{app:proof_transition_stability}
\begin{proof}
We first establish that the value function $V_i^\pi$ is bounded. By Assumption~\ref{assumption:structural_assumptions}, the reward function $\mathcal{R}_i$ is uniformly bounded by $M_i$
\begin{equation*}
\begin{split}
    |V^\pi_i(s)| & = \left| \bb{E}_\pi \left[ \sum_{n=0}^\infty \gamma^n \mathcal{R}_i(s_n, a_n) \mid s_0 = s\right] \right|\\
    & \le \bb{E}_\pi \left[ \sum_{n=0}^\infty \gamma^n |\mathcal{R}_i(s_n, a_n)| \mid s_0 = s\right] \\
    & \le \bb{E}_\pi \left[ \sum_{n=0}^\infty \gamma^n M_i \mid s_0 = s\right] \\
    & = M_i \frac{1}{1 - \gamma} \\
    & < +\infty \\
\end{split}
\end{equation*}
for $\gamma \in [0, 1)$.

We now compare the value functions induced by the two transition models.
For brevity, write
\[
V_{i,f}^\pi(s)
\coloneqq
V_i^\pi(s;f,\Sigma_f),
\qquad
V_{i,g}^\pi(s)
\coloneqq
V_i^\pi(s;g,\Sigma_g).
\]
Their Bellman equations are
\[
V_{i,f}^\pi(s)
=
\int_A \pi(da\mid s)
\left[
\mathcal R_i(s,a)
+
\gamma
\int_S
V_{i,f}^\pi(s')
\mathcal T_f(ds'\mid s,a)
\right],
\]
and
\[
V_{i,g}^\pi(s)
=
\int_A \pi(da\mid s)
\left[
\mathcal R_i(s,a)
+
\gamma
\int_S
V_{i,g}^\pi(s')
\mathcal T_g(ds'\mid s,a)
\right].
\]
Subtracting these equations, the immediate reward terms cancel, yielding
\[
\begin{aligned}
V_{i,f}^\pi(s)-V_{i,g}^\pi(s)
=
\gamma\int_A \pi(da\mid s)
\Bigg[
&\int_S
V_{i,f}^\pi(s')
\mathcal T_f(ds'\mid s,a) - \int_S
V_{i,g}^\pi(s')
\mathcal T_g(ds'\mid s,a)
\Bigg].
\end{aligned}
\]
Adding and subtracting
\[
\int_S
V_{i,g}^\pi(s')
\mathcal T_f(ds'\mid s,a)
\]
gives
\[
\begin{aligned}
V_{i,f}^\pi(s)-V_{i,g}^\pi(s)
=
&\,
\gamma\int_A \pi(da\mid s)
\int_S
\bigl(V_{i,f}^\pi(s')-V_{i,g}^\pi(s')\bigr)
\mathcal T_f(ds'\mid s,a)
\\
&+
\gamma\int_A \pi(da\mid s)
\int_S
V_{i,g}^\pi(s')
\bigl(\mathcal T_f-\mathcal T_g\bigr)
(ds'\mid s,a).
\end{aligned}
\]
Taking absolute values and using
\[
\|V_{i,g}^\pi\|_\infty
\le
\frac{M_i}{1-\gamma},
\]
we obtain
\[
\begin{aligned}
|V_{i,f}^\pi(s)-V_{i,g}^\pi(s)|
\le
&\,
\gamma
\|V_{i,f}^\pi-V_{i,g}^\pi\|_\infty
\\
&+
\frac{\gamma M_i}{1-\gamma}
\sup_{(s,a)\in S\times A}
\int_S
\left|
\mathcal T_f-\mathcal T_g
\right|
(ds'\mid s,a).
\end{aligned}
\]
Taking the supremum over $s\in S$ yields
\[
\|V_{i,f}^\pi-V_{i,g}^\pi\|_\infty
\le
\gamma
\|V_{i,f}^\pi-V_{i,g}^\pi\|_\infty
+
\frac{\gamma M_i}{1-\gamma}
D,
\]
where
\[
D
\coloneqq
\sup_{(s,a)\in S\times A}
\int_S
\left|
\mathcal T_f-\mathcal T_g
\right|
(ds'\mid s,a).
\]
Hence,
\[
\|V_{i,f}^\pi-V_{i,g}^\pi\|_\infty
\le
\frac{\gamma M_i}{(1-\gamma)^2}
D.
\]
Now, using the convention
\[
d_{\mathrm{TV}}(P,Q)
=
\frac12\int |dP-dQ|,
\]
Pinsker's inequality gives
\[
\int |dP-dQ|
\le
\sqrt{2D_{\mathrm{KL}}(P\|Q)}.
\]
Therefore,
\[
D
\le
\sqrt{2}
\sup_{(s,a)\in S\times A}
\sqrt{
D_{\mathrm{KL}}
\left(
\mathcal N(f(s,a),\Sigma_f(s, a))
\,\middle\|\,
\mathcal N(g(s,a),\Sigma_g(s, a))
\right)
}.
\]
Combining the two bounds,
\[
\|V_{i,f}^\pi-V_{i,g}^\pi\|_\infty
\le
\frac{\sqrt{2}\,\gamma M_i}{(1-\gamma)^2}
\sup_{(s,a)\in S\times A}
\sqrt{
D_{\mathrm{KL}}
\left(
\mathcal N(f(s,a),\Sigma_f(s, a))
\,\middle\|\,
\mathcal N(g(s,a),\Sigma_g(s, a))
\right)
}.
\]
\end{proof}

\subsection{Proof of Lemma~\ref{lemma:lemma2}}
\label{app:proof_policy_stability}
\begin{proof}
For every \(s\in S\),
\[
    V_i^\pi(s)
    =
    \int_A Q_i^\pi(s,a)\,\pi(da\mid s),
    \qquad
    V_i^\rho(s)
    =
    \int_A Q_i^\rho(s,a)\,\rho(da\mid s).
\]
Adding and subtracting
\(\int_A Q_i^\pi(s,a)\rho(da\mid s)\) yields
\[
\begin{aligned}
    |V_i^\pi(s)-V_i^\rho(s)|
    &\le
    \left|
        \int_A
        Q_i^\pi(s,a)
        (\pi-\rho)(da\mid s)
    \right| \\
    &\quad
    +
    \int_A
    |Q_i^\pi(s,a)-Q_i^\rho(s,a)|
    \rho(da\mid s).
\end{aligned}
\]

By the Bellman equation,
\[
    Q_i^\pi(s,a)
    =
    \mathcal{R}_i(s,a)
    +
    \gamma
    \int_S
    V_i^\pi(s')
    \mathcal{T}(ds'\mid s,a),
\]
and similarly for \(Q_i^\rho\). Hence,
\[
    Q_i^\pi(s,a)-Q_i^\rho(s,a)
    =
    \gamma
    \int_S
    \bigl(
        V_i^\pi(s')-V_i^\rho(s')
    \bigr)
    \mathcal{T}(ds'\mid s,a).
\]
Therefore,
\[
\begin{aligned}
    |Q_i^\pi(s,a)-Q_i^\rho(s,a)|
    &\le
    \gamma
    \int_S
    |V_i^\pi(s')-V_i^\rho(s')|
    \mathcal{T}(ds'\mid s,a) \\
    &\le
    \gamma
    \left\|V_i^\pi-V_i^\rho\right\|_\infty.
\end{aligned}
\]
Consequently,
\[
\begin{aligned}
    |V_i^\pi(s)-V_i^\rho(s)|
    &\le
    \left|
        \int_A
        Q_i^\pi(s,a)
        (\pi-\rho)(da\mid s)
    \right|
    +
    \gamma
    \left\|V_i^\pi-V_i^\rho\right\|_\infty.
\end{aligned}
\]

For any bounded measurable function \(f\),
\[
    \left|
        \int_A f(a)(\mu-\nu)(da)
    \right|
    \le
    \|f\|_\infty
    \int_A |\mu-\nu|(da)
    =
    2\|f\|_\infty d_{\mathrm{TV}}(\mu,\nu).
\]
It follows that
\[
\begin{aligned}
    \left|
        \int_A
        Q_i^\pi(s,a)
        (\pi-\rho)(da\mid s)
    \right|
    &\le
    2
    \|Q_i^\pi\|_\infty
    d_{\mathrm{TV}}
    \bigl(
        \pi(\cdot\mid s),
        \rho(\cdot\mid s)
    \bigr) \\
    &\le
    2
    \|Q_i^\pi\|_\infty
    d_{\mathrm{TV}}(\pi,\rho).
\end{aligned}
\]

Moreover, since
\(|\mathcal{R}_i(s,a)|\le M_i\),
\[
\begin{aligned}
    |Q_i^\pi(s,a)|
    &=
    \left|
        \mathbb{E}_\pi
        \left[
            \sum_{k=0}^\infty
            \gamma^k
            \mathcal{R}_i(s_k,a_k)
            \,\middle|\,
            s_0=s,\ a_0=a
        \right]
    \right| \\
    &\le
    \mathbb{E}_\pi
    \left[
        \sum_{k=0}^\infty
        \gamma^k
        |\mathcal{R}_i(s_k,a_k)|
        \,\middle|\,
        s_0=s,\ a_0=a
    \right] \\
    &\le
    M_i
    \sum_{k=0}^\infty \gamma^k
    =
    \frac{M_i}{1-\gamma}.
\end{aligned}
\]
Thus,
\[
    \|Q_i^\pi\|_\infty
    \le
    \frac{M_i}{1-\gamma}.
\]
Combining the preceding estimates gives, for every \(s\in S\),
\[
    |V_i^\pi(s)-V_i^\rho(s)|
    \le
    \frac{2M_i}{1-\gamma}
    d_{\mathrm{TV}}(\pi,\rho)
    +
    \gamma
    \left\|V_i^\pi-V_i^\rho\right\|_\infty.
\]
Taking the supremum over \(s\in S\) yields
\[
    \left\|V_i^\pi-V_i^\rho\right\|_\infty
    \le
    \frac{2M_i}{1-\gamma}
    d_{\mathrm{TV}}(\pi,\rho)
    +
    \gamma
    \left\|V_i^\pi-V_i^\rho\right\|_\infty.
\]
Rearranging,
\[
    (1-\gamma)
    \left\|V_i^\pi-V_i^\rho\right\|_\infty
    \le
    \frac{2M_i}{1-\gamma}
    d_{\mathrm{TV}}(\pi,\rho),
\]
and therefore
\[
    \left\|V_i^\pi-V_i^\rho\right\|_\infty
    \le
    \frac{2M_i}{(1-\gamma)^2}
    d_{\mathrm{TV}}(\pi,\rho).
\]
\end{proof}

\subsection{Proof of Theorem~\ref{theorem:convergence_rate}}
\label{app:proof_convergence_rate}
\begin{proof}
We first decompose the error into a policy-approximation term and a
transition-model term. By the triangle inequality,
\begin{align}
&
\left\|
V_i^{\rho^\ell}(\cdot;\mathcal{T}_\ell)
-
V_i^{\rho^*}(\cdot;\mathcal{T})
\right\|_\mathcal{C}
\nonumber\\
&\qquad\le
\left\|
V_i^{\rho^\ell}(\cdot;\mathcal{T}_\ell)
-
V_i^{\rho^*}(\cdot;\mathcal{T}_\ell)
\right\|_\mathcal{C}
+
\left\|
V_i^{\rho^*}(\cdot;\mathcal{T}_\ell)
-
V_i^{\rho^*}(\cdot;\mathcal{T})
\right\|_\mathcal{C}.
\label{eq:convergence_decomposition}
\end{align}
Applying Lemma~\ref{lemma:lemma2} to the first term, with the
transition kernel $\mathcal{T}_\ell$ held fixed, gives
\[
\left\|
V_i^{\rho^\ell}(\cdot;\mathcal{T}_\ell)
-
V_i^{\rho^*}(\cdot;\mathcal{T}_\ell)
\right\|_\mathcal{C}
\le
\frac{2M_i}{(1-\gamma)^2}
\Delta_\ell^\rho.
\]
For the second term, Lemma~\ref{lemma:lemma1}, applied with the joint
policy $\rho^*$ held fixed, yields
\[
\left\|
V_i^{\rho^*}(\cdot;\mathcal{T}_\ell)
-
V_i^{\rho^*}(\cdot;\mathcal{T})
\right\|_\mathcal{C}
\le
\frac{\sqrt{2}\gamma M_i}{(1-\gamma)^2}
\eta_\ell.
\]
Combining the two inequalities proves
\[
\left\|
V_i^{\rho^\ell}(\cdot;\mathcal{T}_\ell)
-
V_i^{\rho^*}(\cdot;\mathcal{T})
\right\|_\mathcal{C}
=
\mathcal O(\Delta_\ell^\rho+\eta_\ell).
\]

It remains to bound $\eta_\ell$ under the Gaussian-process model.
Let
\[
\mathcal E_\alpha
\coloneqq
\left\{
f^*
\in
\bigcap_{\ell\geq0}
\mathcal M_\ell(\alpha)
\right\}.
\]
By Lemma~\ref{lemma:well_calibrated_gp},
\[
\mathbb P(\mathcal E_\alpha)\geq1-\alpha.
\]
We henceforth work on the event $\mathcal E_\alpha$ and $\mathcal{C} \times A$. By the definition
of the confidence sets, for every $\ell\geq0$, $(s,a)\in \mathcal{C} \times A$,
and $j\in[n_S]$,
\begin{equation}
\left|
\mu_{\ell,j}(s,a)-f_j^*(s,a)
\right|
\le
\beta_\ell(\alpha)\sigma_{\ell,j}(s,a).
\label{eq:gp_confidence_event}
\end{equation}

For brevity, fix $z=(s,a)$ and define
\[
r_{\ell,j}(z)
\coloneqq
\frac{\sigma_{\ell,j}^2(z)}{\sigma^2}.
\]
Since,
\[
\widehat\Sigma_\ell(z)
=
\operatorname{diag}
\left(
\sigma^2+\sigma_{\ell,1}^2(z),
\dots,
\sigma^2+\sigma_{\ell,n_S}^2(z)
\right),
\]
the KL divergence between the true and learned Gaussian transition
distributions therefore satisfies
\begin{align}
&
D_{\mathrm{KL}}
\left(
\mathcal{T}(\cdot\mid z)
\,\middle\|\,
\mathcal{T}_\ell(\cdot\mid z)
\right)
\nonumber\\
&\quad=
\frac12
\sum_{j=1}^{n_S}
\left[
\log(1+r_{\ell,j}(z))
-
\frac{r_{\ell,j}(z)}{1+r_{\ell,j}(z)}
+
\frac{
\left(\mu_{\ell,j}(z)-f_j^*(z)\right)^2
}{
\sigma^2+\sigma_{\ell,j}^2(z)
}
\right].
\label{eq:kl_gp_expansion}
\end{align}
Using \eqref{eq:gp_confidence_event},
\[
\frac{
\left(\mu_{\ell,j}(z)-f_j^*(z)\right)^2
}{
\sigma^2+\sigma_{\ell,j}^2(z)
}
\le
\beta_\ell^2(\alpha)
\frac{r_{\ell,j}(z)}{1+r_{\ell,j}(z)}
\le
\beta_\ell^2(\alpha)r_{\ell,j}(z).
\]
Moreover, for every $r\geq0$,
\[
\log(1+r)-\frac{r}{1+r}
\le
\frac{r^2}{2}.
\]
Hence
\[
D_{\mathrm{KL}}
\left(
\mathcal{T}(\cdot\mid z)
\,\middle\|\,
\mathcal{T}_\ell(\cdot\mid z)
\right)
\le
\sum_{j=1}^{n_S}
\left[
\frac{r_{\ell,j}^2(z)}{4}
+
\frac{\beta_\ell^2(\alpha)r_{\ell,j}(z)}{2}
\right].
\]
By the definition of $q_\ell$,
\[
r_{\ell,j}(z)
\le
\frac{q_\ell}{\sigma^2},
\]
and therefore
\[
D_{\mathrm{KL}}
\left(
\mathcal{T}(\cdot\mid z)
\,\middle\|\,
\mathcal{T}_\ell(\cdot\mid z)
\right)
\le
n_S
\left(
\frac{q_\ell^2}{4\sigma^4}
+
\frac{\beta_\ell^2(\alpha)q_\ell}{2\sigma^2}
\right).
\]
Taking the supremum over $z\in \mathcal{C} \times A$ yields
\[
\eta_\ell^2
\le
n_S
\left(
\frac{q_\ell^2}{4\sigma^4}
+
\frac{\beta_\ell^2(\alpha)q_\ell}{2\sigma^2}
\right).
\]
Since Lemma~\ref{lemma:well_calibrated_gp} implies
\[
\beta_\ell(\alpha)
=
\mathcal O(\sqrt{\Gamma_\ell}),
\]
we obtain
\[
\eta_\ell
=
\mathcal O\!\left(
\sqrt{n_S}
\left[
\frac{q_\ell}{\sigma^2}
+
\frac{\sqrt{\Gamma_\ell q_\ell}}{\sigma}
\right]
\right).
\]
Substitution into \eqref{eq:convergence_decomposition} proves the
high-probability bound.

Thus, every term in the value bound converges whenever
$\Delta_\ell^\rho\to0$ and $\eta_\ell\to0$. In particular, this occurs
under the GP conditions
\[
\Delta_\ell^\rho\to0,
\qquad
q_\ell\to0,
\qquad
\Gamma_\ell q_\ell\to0,
\]
which imply $\eta_\ell\to0$ on the all-time confidence event.

We now show that $\rho^*$ is an $\varepsilon$-Nash equilibrium on $\mathcal C$ in the sense of Proposition~\ref{prop:approximation} under the estimator-agnostic conditions
\[
\Delta_\ell^\rho\to0,
\qquad
\eta_\ell\to0.
\]

By~\cite[Theorem~14]{liese2006divergences}, it follows that
\begin{equation*}
\begin{split}
d_{\mathrm{TV}}
\left(
  (\pi_i,\rho_{-i}^\ell)(\cdot\mid s),
  (\pi_i,\rho_{-i}^*)(\cdot\mid s)
\right)
&\leq
d_{\mathrm{TV}}
\left(
  \rho_{-i}^\ell(\cdot\mid s),
  \rho_{-i}^*(\cdot\mid s)
\right) \\
&\leq
d_{\mathrm{TV}}
\left(
  \rho^\ell(\cdot\mid s),
  \rho^*(\cdot\mid s)
\right) \\
&\leq \Delta_\ell^\rho.
\end{split}
\end{equation*}

Restricting attention to $\mathcal{C}$, set
\[
    M^* = \max_{i \in [n]} \frac{2M_i}{(1 - \gamma)^2},
\]
and let $\Pi$ denote the set of all admissible policies of the original game. It then follows that for $\ell$ large enough (see Proposition~\ref{prop:approximation})
\begin{equation*}
\begin{split}
V_i^{\rho^\ell}(s;\mathcal{T})
&\geq
V_i^{\rho^\ell}(s;\mathcal{T}_\ell)
-
\left|
V_i^{\rho^\ell}(s;\mathcal{T})
-
V_i^{\rho^\ell}(s;\mathcal{T}_\ell)
\right| \\
&\geq
V_i^{(\pi_i,\rho_{-i}^\ell)}(s;\mathcal{T}_\ell)
-\varepsilon
-
\left|
V_i^{\rho^\ell}(s;\mathcal{T})
-
V_i^{\rho^\ell}(s;\mathcal{T}_\ell)
\right| \\
&\geq
V_i^{(\pi_i,\rho_{-i}^\ell)}(s;\mathcal{T})
-\varepsilon \\
&
-
\left|
V_i^{\rho^\ell}(s;\mathcal{T})
-
V_i^{\rho^\ell}(s;\mathcal{T}_\ell)
\right| \\
& 
-
\left|
V_i^{(\pi_i,\rho_{-i}^\ell)}(s;\mathcal{T})
-
V_i^{(\pi_i,\rho_{-i}^\ell)}(s;\mathcal{T}_\ell)
\right| \\
& \geq V_i^{(\pi_i,\rho_{-i}^\ell)}(s;\mathcal{T})
-\varepsilon - 2\max_{i \in [n]}\, \sup_{\pi \in \Pi}
\| V_i^\pi(\cdot ; \mathcal{T}) - V_i^\pi(\cdot ; \mathcal{T}_\ell)\|_{\mathcal{C}}. \\
\end{split}
\end{equation*}
Thus, 
\begin{equation*}
\begin{split}
V_i^{\rho^*}(s; \mathcal{T})
&\geq
V_i^{\rho^\ell}(s; \mathcal{T}) - M^* \Delta_\ell^\rho \qquad (\text{Lemma~\ref{lemma:lemma2}})\\
&\geq
V_i^{\pi_i, \rho_{-i}^\ell}(s; \mathcal{T})
- \varepsilon
- 2\max_{i \in [n]}\, \sup_{\pi \in \Pi}
\| V_i^\pi(\cdot ; \mathcal{T}) - V_i^\pi(\cdot ; \mathcal{T}_\ell)\|_{\mathcal{C}}
- M^*\Delta_\ell^\rho \\
&\geq
V_i^{\pi_i, \rho_{-i}^*}(s; \mathcal{T})
- \varepsilon
- 2\max_{i \in [n]}\, \sup_{\pi \in \Pi}
\| V_i^\pi(\cdot ; \mathcal{T}) - V_i^\pi(\cdot ; \mathcal{T}_\ell)\|_{\mathcal{C}}
- 2M^*\Delta_\ell^\rho.
\end{split}
\end{equation*}
Finally, Lemma~\ref{lemma:lemma1} gives a policy-uniform bound on
\[
    \| V_i^\pi(\cdot ; \mathcal{T}) - V_i^\pi(\cdot ; \mathcal{T}_\ell)\|_{\mathcal{C}}
\]
by a constant multiple of $\eta_\ell$. Hence $\eta_\ell\to0$, together
with $\Delta_\ell^\rho\to0$, yields the desired result upon letting
$\ell\to+\infty$. The GP conditions above imply this same conclusion
through their high-probability guarantee that $\eta_\ell\to0$ if  $q_\ell\to0$,
and $\Gamma_\ell q_\ell\to0$.
\end{proof}

\section{Additional Experimental Details}
\label{app:experimental_details}

This appendix records the reference-solution, kernel-calibration, and evaluation-configuration details underlying the controlled study in Section~\ref{sec:experiments}.

\subsection{Reference solution and evaluation target}
Because the transition dynamics are component-wise (agent $i$'s action does not affect agent $j$'s state) the benchmark admits a coordinate-wise Nash equilibrium policy solution in terms of $n$ identical one-dimensional control problems. This structure provides a high-accuracy reference policy for evaluating the learned policies. It suffices to solve
\[
    R(s, a) \coloneqq -\min \{s^2 + a^2, 4 \}, \quad (s, a) \in \mathbb{R} \times [-1, 1],
\]
whose value under a policy $\pi$ is
\[
    V^\pi(s)
    \coloneqq
    \int_{-1}^1\pi(da\mid s)
    \left[
        R(s,a)
        +\gamma\int_{-\infty}^{\infty}
        V^\pi(s^\prime)\,\mathcal{T}(ds^\prime\mid s,a)
    \right],
\]
and whose optimal value function satisfies
\[
    V^*(s) = \max_{a \in [-1, 1]} \left[ R(s, a) + \gamma \int_{-\infty}^\infty V^*(s^\prime) \mathcal{T}(ds^\prime \mid s, a) \right].
\]

Figure~\ref{fig:value_function_policy} shows the numerical approximation of the optimal value function and policy for $\gamma=0.9$, the discount factor used throughout the experiments. Combining the one-dimensional optimal policies yields a Nash equilibrium of the full benchmark game under the true dynamics $f^*$ and noise standard deviation $\sigma=0.1$: no agent can improve its value through a unilateral deviation. Because the reward is shared, this equilibrium has a natural cooperative interpretation in which every agent minimizes its contribution to the common cost.

\begin{figure}[h]
    \centering
    \includegraphics[width=\linewidth]{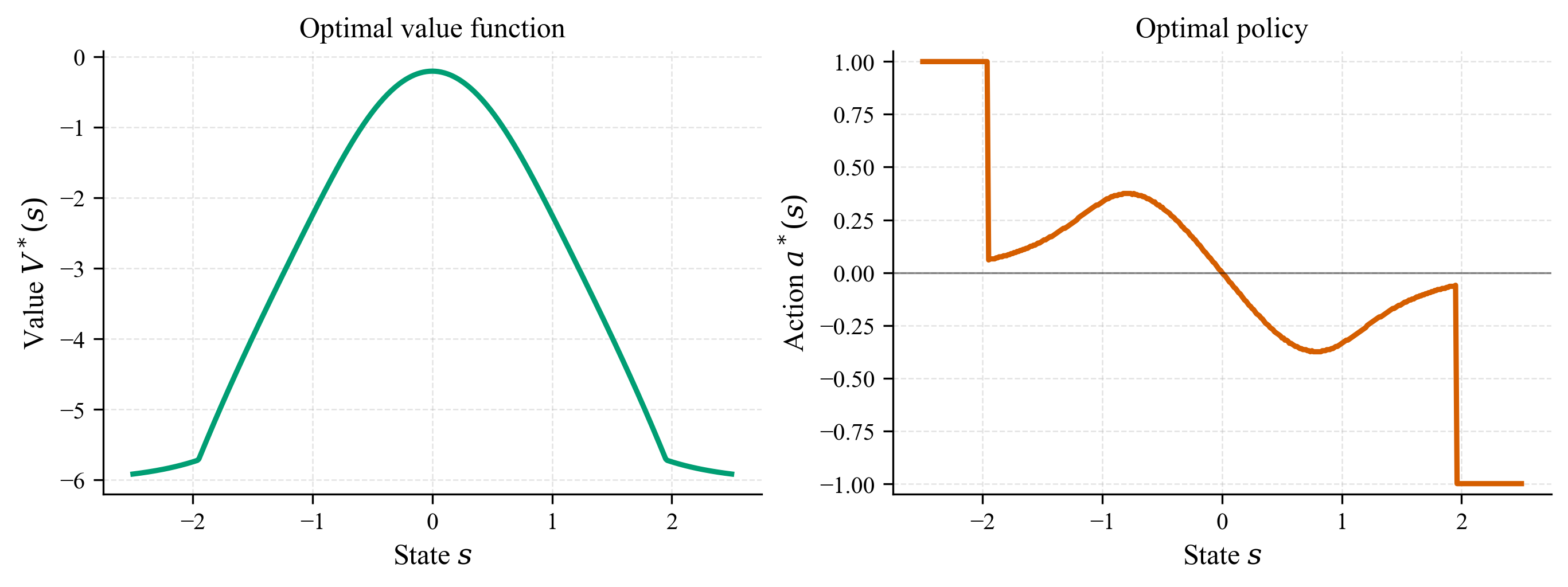}
    \caption{
        Numerical solution of the one-dimensional benchmark stochastic control problem used in our experiments. The left panel shows the optimal value function and the right panel shows the optimal policy obtained via value iteration.
    }
    \label{fig:value_function_policy}
\end{figure}

\begin{remark}[Policy Discontinuity]
    The discontinuity visible in Figure~\ref{fig:value_function_policy} is a direct consequence of the reward structure. When $|s| \ge 2$, the reward saturates at
    \[
        \mathcal{R}(s,a) = -4,
    \]
    regardless of the action taken. Beyond this point, the immediate reward no longer distinguishes between actions, and the optimal decision is determined entirely by future rewards. Consequently, the controller selects the most aggressive admissible action to drive the state toward the origin as rapidly as possible, yielding
    \[
        a^* = -1 \quad \text{for } s \ge 2,
        \qquad
        a^* = 1 \quad \text{for } s \le -2.
    \]
\end{remark}

\subsection{Kernel calibration}
We select the hyperparameters by maximizing the Gaussian marginal likelihood of the observed transition data, using a shared kernel across state coordinates as required by Assumption~\ref{assumption:rkhs_regularity_assumption}.

Let
\[
    \mathcal{D}_{\mathrm{ker}}
    =
    \{(z_r,y_r)\}_{r=1}^{N_{\mathrm{ker}}},
    \qquad
    z_r \coloneqq (s_r,a_r),
\]
where $y_r=s_{r+1}\in\mathbb{R}^{n_S}$ is the observed next state. For each coordinate $j\in\{1,\ldots,n_S\}$, $f_j^*$ denotes the fixed true dynamics, and we use the regression model
\[
    f_j\sim\mathcal{GP}(0,k_\theta),
    \qquad
    y_r^{(j)}=f_j^*(z_r)+u_r^{(j)},
    \qquad
    u_r^{(j)}\sim\mathcal{N}(0,\lambda),
\]
with independence across observations and coordinates. Here, $\lambda$ is the GP likelihood-noise parameter, whereas $\widehat\sigma^2$ is the plug-in variance added to the planner's posterior predictive covariance in Algorithm~\ref{alg:finite_game_rl}. Let
\[
    K_\theta
    =
    \bigl[k_\theta(z_r,z_q)\bigr]_{r,q=1}^{N_{\mathrm{ker}}},
    \qquad
    C_\theta=K_\theta+\lambda I_{N_{\mathrm{ker}}}.
\]
Writing $y^{(j)}=(y_1^{(j)},\ldots,y_{N_{\mathrm{ker}}}^{(j)})^\top$ and $Z=(z_1,\ldots,z_{N_{\mathrm{ker}}})$, the log marginal likelihood for coordinate $j$ is
\[
    \log p(y^{(j)} \mid Z,\theta)
    =
    -\frac{1}{2} (y^{(j)})^\top C_\theta^{-1} y^{(j)}
    -\frac{1}{2}\log \det C_\theta
    -\frac{N_{\mathrm{ker}}}{2}\log(2\pi).
\]
The shared hyperparameters are selected jointly across all output coordinates:
\[
    \theta^\star
    \in
    \arg\min_{\theta \in \Theta}
    \left\{
        \frac{1}{2}\sum_{j=1}^{n_S}
        (y^{(j)})^\top C_\theta^{-1} y^{(j)}
        +
        \frac{n_S}{2}\log \det C_\theta
    \right\},
\]
where constants independent of \(\theta\) have been dropped.

In the experiments, this optimization is performed separately for the RBF and Matérn kernels. For the RBF kernel, we use
\[
    k_\theta(z,z')
    =
    \sigma_f^2
    \exp\left(
        -\frac{\|z-z'\|_2^2}{2\ell_{\mathrm{ker}}^2}
    \right),
\]
with hyperparameters
\[
    \theta = (\sigma_f^2,\ell_{\mathrm{ker}},\lambda).
\]
For the Matérn kernel, we use
\[
    k_\theta(z,z')
    =
    \sigma_f^2
    \frac{2^{1-\nu}}{\Gamma(\nu)}
    \left(
        \frac{\sqrt{2\nu}\|z-z'\|_2}{\ell_{\mathrm{ker}}}
    \right)^\nu
    K_\nu\left(
        \frac{\sqrt{2\nu}\|z-z'\|_2}{\ell_{\mathrm{ker}}}
    \right),
\]
with hyperparameters
\[
    \theta = (\sigma_f^2,\ell_{\mathrm{ker}},\lambda,\nu),
\]
where $K_\nu$ is the modified Bessel function of the second kind. The optimization domain $\Theta$ enforces $\sigma_f^2\le\kappa$ and positivity of the remaining scale parameters. After optimizing the marginal likelihood, we fix $k=k_{\theta^\star}$ for the corresponding GP-MARL run. Thus, both kernels satisfy the bounded-diagonal requirement, while their distinct smoothness structures induce different posterior variances $q_\ell$ and information-gain terms $\Gamma_\ell$ -- the kernel-dependent quantities in Theorem~\ref{theorem:convergence_rate}.

\subsection{Evaluation configurations}
Figure~\ref{fig:optimized_2player} reports results for the two-player setting. Kernel hyperparameters are selected once from a design set sampled uniformly over $[-0.8,0.8]^2\times[-0.5,0.5]^2$; after this calibration step, every posterior and policy update uses the single interaction trajectory collected by GP-MARL. The finite game is constructed on $\mathcal{C}=[-0.8,0.8]^2$ with an active joint action grid over $[-0.5,0.5]^2$ and mesh size $\delta_G=0.05$. The one-dimensional reference solution in Figure~\ref{fig:value_function_policy} uses the full action space $[-1,1]$, inducing a two-agent reference policy on $[-1,1]^2$.

For numerical efficiency, policy and kernel updates are performed every
\[
    H_\ell = \min\{2^\ell,30\}
\]
environment interactions.

We fix the variance estimate at $\widehat\sigma^2=0$, so the finite-game solver uses the learned covariance $\Sigma_\ell$. This choice isolates epistemic GP uncertainty in the numerical study, while the environment continues to evolve with noise variance $\sigma^2=0.1^2$.

To evaluate convergence, let $t_\ell$ denote the environment time at update $\ell$. We define the pointwise value error as
\[
e_\ell
\coloneqq
\left|
V_1^{\rho^\ell}(s_{t_\ell};\widehat{\mathcal{T}}_\ell)
-
V_1^*(s_{t_\ell};\mathcal{T})
\right|,
\]
and report its natural logarithm, $\ln e_\ell$, where $V_1^*(\cdot;\mathcal{T})$ is computed from the reference policy in Figure~\ref{fig:value_function_policy}. This metric is a trajectory-level analogue of the compact-set value discrepancy in Theorem~\ref{theorem:convergence_rate}. All value functions are computed numerically using Gauss--Hermite quadrature.

For the agent-scaling experiments in Figure~\ref{fig:varying_players}, hyperparameters are optimized independently for $j\in\{2,3\}$ agents using design sets sampled uniformly from $[-0.5,0.5]^j\times[-0.5,0.5]^j$. The finite games use mesh size $\delta_G=0.1$, and the transition model and policies are updated every $H=123$ environment interactions. As above, the variance estimate is set to $\widehat\sigma^2=0$.

\clearpage
\bibliography{references}

\end{document}